\documentclass{article}

    \usepackage[preprint,nonatbib]{neurips_2020}

\usepackage[utf8]{inputenc} 
\usepackage[T1]{fontenc}    
\usepackage{hyperref,float}       
\usepackage{url}            
\usepackage{booktabs}       
\usepackage{amsfonts}       
\usepackage{nicefrac}       
\usepackage{microtype}      

\usepackage[ruled,vlined]{algorithm2e}
\usepackage[utf8]{inputenc}
\usepackage[T1]{fontenc}
\usepackage{charter}
\usepackage{amsmath}
\usepackage{amssymb}
\usepackage{mathrsfs} 
\usepackage{mathtools} 
\usepackage{epic,eepic}
\usepackage{enumitem}
\usepackage{theorem}
\usepackage{appendix}
\PassOptionsToPackage{normalem}{ulem}
\usepackage{pifont}  
\usepackage{euscript}
\usepackage{calc}
\usepackage{tikz}
\usepackage{graphicx,wrapfig}
\usepackage[titles]{tocloft}
\usepackage{tikz}
\usetikzlibrary{arrows}
\usepackage[english]{babel}
\newcommand{\email}[1]{\href{mailto:#1}{\nolinkurl{#1}}}
\usepackage[normalem]{ulem}     

\newtheorem{theorem}{Theorem}[section]
\newtheorem{lemma}[theorem]{Lemma}

\newtheorem{proposition}[theorem]{Proposition}
\newtheorem{assumption}[theorem]{Assumption}
\theoremstyle{plain}{\theorembodyfont{\rmfamily}%
}
\theoremstyle{plain}{\theorembodyfont{\rmfamily}%
}
\theoremstyle{plain}{\theorembodyfont{\rmfamily}%
\newtheorem{remark}[theorem]{Remark}}
\theoremstyle{plain}{\theorembodyfont{\rmfamily}%
\theoremstyle{plain}{\theorembodyfont{\rmfamily}%
}
\theoremstyle{plain}{\theorembodyfont{\rmfamily}%
}
\theoremstyle{plain}{\theorembodyfont{\rmfamily}%
\newtheorem{definition}[theorem]{Definition}}
\theoremstyle{plain}{\theorembodyfont{\rmfamily}%
}

\numberwithin{equation}{section}

\newcommand{\scal}[2]{{\left\langle{{#1},{#2}}\right\rangle}}

\newcommand{\menge}[2]{\big\{{#1}~\big |~{#2}\big\}}

\newcommand{\YY}{\ensuremath{{\mathcal Y}}}

\newcommand{\epi}{\ensuremath{\text{\rm epi }}}

\newcommand{\RR}{\ensuremath{\mathbb{R}}}
\newcommand{\R}{\ensuremath{\mathbb{R}}}

\newcommand{\RPP}{\ensuremath{\left]0,+\infty\right[}}

\newcommand{\NN}{\ensuremath{\mathbb N}}

\newcommand{\exi}{\ensuremath{\exists\,}}
\newcommand{\pinf}{\ensuremath{{+\infty}}}

\newcommand{\Sum}{\ensuremath{\displaystyle\sum}}

\newcommand{\zeroun}{\ensuremath{\left]0,1\right[}}   
\newcommand{\rzeroun}{\ensuremath{\left]0,1\right]}}

\newcommand{\gra}{\ensuremath{\text{\rm graph}\,}}

\newcommand{\argmin}{\ensuremath{\text{\rm argmin}\,}}
\newcommand{\argmax}{\ensuremath{\text{\rm argmax}\,}}

\newenvironment{proofof}[1]
{\textbf{Proof of #1}
}
{
}

\title{A H\"olderian backtracking method for min-max and min-min problems}

\author{%
  Jérôme Bolte\thanks{Authors are listed in alphabetical order} \\
  Toulouse School of Economics\\
  University of Toulouse\\
  \texttt{jerome.bolte@ut-capitole.fr} \\
   \And
   Lilian Glaudin\footnote[1]{} \\
  ANITI \\
  University of Toulouse \\
   \texttt{lilian@glaudin.net} \\
  \AND
  Edouard Pauwels\footnote[1]{} \\
  IRIT \\
  Unversity of Toulouse \\
   \texttt{edouard.pauwels@irit.fr} \\
  \And
  Mathieu Serrurier\footnote[1]{} \\
  IRIT \\
  University of Toulouse \\
   \texttt{mathieu.serrurier@irit.fr} \\
}

\begin{document}

\maketitle

\begin{abstract}
We present a new algorithm to solve min-max or min-min problems out of
the convex world. We use rigidity assumptions, ubiquitous in learning,
making our method applicable to many optimization problems. Our approach
takes advantage of hidden regularity properties and allows us to devise
a simple algorithm of ridge type. An original feature of our method is
to come with automatic step size adaptation which departs from the usual
overly cautious backtracking methods. In a general framework, we provide
convergence theoretical guarantees and rates. We apply our findings on
simple GAN problems obtaining promising numerical results.
\end{abstract}

\section{Introduction}

Adversarial learning, introduced in \cite{Good14}, see also \cite{Arjo17}, calls for the development of algorithms  addressing large scale, smooth problems of the type
\begin{align}\label{basicprob}
\min_{x\in\R^d }\max_{y\in\YY} L(x,y), 
\end{align}
 where $\YY$ is a constraint set, and $L$ is a given cost function. 
 This structure happens to be ubiquitous in optimization and game theory, but generally under assumptions that are not those met in learning. In optimization it stems from the Lagrangian approach and duality theory, see e.g., \cite{Boyd04,bertsekas2014constrained,Comb17a}, while in game theory it comes from  zero-sum 2-players games, see e.g., \cite{neumann1928theorie,Von49,laraki2019mathematical}. 
 Dynamics for addressing \eqref{basicprob} have thus naturally two types. They may be built on strategic considerations, so that  algorithms correspond to a sequence of actions chosen by antagonistic players, see \cite{laraki2019mathematical} and references therein. In general these methods are not favorable to optimization because the contradictory interests of players induce oscillations and slowness in the identification of optimal strategies. Optimization algorithms seem more interesting for our purposes because they focus on the final result, i.e., finding an optimal choice $x$, regardless of the adversarial strategy issues. In that respect, there are two possibilities: the variational inequality approach which treat  minimization and maximization variables on an equal footing, see e.g. \cite{Korp76,Nemi04,Comb17a} or \cite{mertikopoulos2018optimistic,hsieh2019convergence,Gide18} in learning. On the other hand, some methods break this symmetry, as  primal  or augmented Lagrangian methods. In those, a large number of explicit steps, implicit steps, or global minimization  are performed on one   variable while the other is updated much more cautiously in an explicit incremental way, see e.g., \cite{bertsekas2014constrained,sabach2019lagrangian}.

Our work is written in this spirit: we assume that the under-max argument is tractable  with a good precision, and we construct our algorithm on the following model:
\begin{align*}\label{algo:intro}
& y_n = \argmax_{y\in\YY} L(x_n,y),\\
& x_{n+1}=x_n-\gamma_n\nabla_x L(x_n,y_n), \gamma_n>0, \,n\geq 0.
\end{align*}
As explained above, the rationale is not new\footnote{It can be traced back to the origin of augmented Lagrangian methods, see e.g.,  \cite{rockafellar1981proximal}}, and is akin to many methods in the literature on learning where the global optimization is performed approximately by multiple gradient steps \cite{Noui19,Noui19a}  or by clever revision steps, as in the "follow the ridge" method, see \cite{Wang19}.

\paragraph{Backtrack H\"older} What is new then?  The surprising fact is that we can provide theoretical grounds to devise {\em large steps} and thus obtain aggressive learning rates  with few assumptions. This is done by exploiting  some hidden properties of the value function $g=\max_y L(\cdot,y)$ under widespread rigidity assumptions. Let us sketch the ideas of our approach. First, under a uniqueness assumption on the maximizer, our method appears to be a gradient method on the value function for ``player I" ( ``the generator" of GANs) 
\[x_{n+1}=x_n-\gamma_n\nabla g(x_n).\]
Secondly, we use the fact that $g$ has a locally H\"olderian gradient\footnote{Recall that $G \colon \RR^d \mapsto \RR^{d'}$ is locally H\"olderian if for all bounded subset $V \subset \RR^d$, there exists $\beta$ and $\nu$ positive such that $\|G(x) - G(y)\| \leq \beta \|x - y\|^\nu$, whenever $x,y \in V$.} whenever $L$ is semi-algebraic or analytic-like, a situation which covers most of the problems met in practice. 
With such observations, we may then develop  automatic learning rate strategies and a diagonal backtracking method, that we call ``Backtrack H\"older methods for min-max".

\paragraph{Contributions} \begin{itemize}

\item We provide  new algorithms whose steps are tuned automatically: Backtrack H\"older gradient and Backtrack H\"older for min-max methods.

\item Our algorithms are shown  to perform with  $O(\epsilon^{-[2+c]})$ rate, where $c$ is a cost incurred by the diagonal backtracking process (which is  negligible in practice), and to provide  general convergence  guarantees to points $(x^*, y^*)$ satisfying $y^*=\argmax_y L(x^*,y)$ and $\nabla_x L(x^*,y^*) = 0$. This is done within a fairly general framework, since  $L$ is merely assumed  semi-algebraic while  the ``best response" of player II is only required to be singled-valued. 

\item A byproduct of our work is a global  convergence result for H\"older methods, which were earlier investigated in the literature \cite{Berg20,Grap19,Nest15,Yash16}.

\item Our work is theoretical in essence. It is merely a first step towards more involved research, regarding the effect of nonsmoothness or stochastic subsampling. We propose however numerical experiments on learning problems. First on the ``Sinkhorn GANs", \cite{pmlr-v84-genevay18a,genevay2017gan}, which rely on optimal transport losses regularized through the addition of an entropic term, and second on Wasserstein GANs \cite{Arjo17} which are a natural extensions of GANs \cite{Good14}.
\end{itemize}

\section{Gradient algorithms \textit{\`a la} H\"older}
\label{sec:holder}
Our results often use semi-algebraic assumptions which are pervasive in optimization and machine learning, see e.g. \cite{castera2019inertial} and references therein.

Our method and proofs  are presented  in view of solving min-max problems,  but the techniques are identical for the min-min case. $\R^d,\R^{d'}$ are endowed with their Euclidean structure.

\subsection{Framework: a single valued best response and semi-algebraicity}
 \label{sec:singleBestRep}
Let $\YY\subset\RR^{d'}$ be a 
 nonempty closed semi-algebraic set, see Definition~\ref{def:sa} in Appendix. 

\paragraph{Properties of the value function and its best response}
\begin{assumption}[Standing assumptions]
\label{a:LC1}
 $L$ is a $C^1$ semi-algebraic function on $\RR^d\times\RR^{d'}$ such that
$(x,y)\mapsto \nabla_x L(x,y)$ is jointly continuous. Furthermore, for any compact sets $K_1\subset\RR^d$ and  $K_2\subset\RR^{d'}$, there exist $\beta_1,\beta_2\in\RPP$ such that, $\forall x_1,x_2\in K_1,\, \forall y_1,y_2\in K_2,$
\begin{equation}
\|\nabla_x L(x_1,y_1)-\nabla_x L(x_2,y_2)\|\leq
\beta_1\|x_1-x_2\|+\beta_2\|y_1-y_2\|.
\end{equation}
\end{assumption}
Borrowing the terminology from game theory, one defines the {\em value function} as  $g(\cdot)=\max_{y\in\YY} L(\cdot,y)$ and
the {\em best response mapping} $p(\cdot)=\argmax_{y\in\YY} L(\cdot,y)$ for $x\in\R^d$. 
\begin{assumption}[Well posedness]\label{ass:psing} 
H1. $p(x)$ is nonempty and single valued for every $x\in\RR^d$, \\
H2. $p$ is continuous.
\end{assumption}
The first part of the assumption is satisfied whenever $L(x,\cdot)$ is strictly concave, see e.g. \cite{Lin20}. Note also that if $L(x,\cdot)$ is concave, as in a dual  optimal transport formulation, some regularization techniques can be used to obtain uniqueness and preserve semi-algebraicity, see e.g., \cite{cuturi2013sinkhorn}.
As for the {\em H2} continuity assumption, it is much less stringent than it may look:
\begin{proposition}[Continuity of $p$]
    \label{p:gdiff0}
    Suppose that Assumption~\ref{a:LC1} and Assumption~\ref{ass:psing}-H1 are satisfied, and that either $\YY$ is compact, or  $p$ is bounded on bounded sets.
    Then the best response $p$ is a continuous function, that is Assumption~\ref{ass:psing}-H2 is fulfilled.
 \end{proposition}

Combining these assumptions with Tarski-Seidenberg theorem and the properties of semi-algebraic functions \cite{bochnak1987geometrie}, we obtain the following. 
\begin{proposition}[Properties of $p$ and $g$]
    \label{p:gdiff}
    Suppose that Assumption~\ref{a:LC1} and Assumption~\ref{ass:psing} are satisfied. Then 
    \begin{enumerate}
        \item \label{pi:gdiff2} $g$ is differentiable and for all  $\bar{x} \in\R^d $,  $\nabla g(\bar{x}) =\nabla_x L(\bar{x},p(\bar{x}))$,
        \item \label{pi:gdiff3} both the value function $g$ and the best response $p$ are semi-algebraic,
        \item \label{pi:gdiff4} the gradient of the value function, $\nabla g$, is locally H\"older.
    \end{enumerate}
\end{proposition}
\begin{remark}
Consider 
$L(x,y)= xy$,
$\YY=[-1,1]$, one sees that 
$g(x)=\max_{y\in[-1,1]}L(x,y) = |x|$ while $p(x)=\mbox{sign } x$ if $x\neq 0$ and $p(0)=[-1,1]$. This shows that Assumption~\ref{ass:psing} is a necessary assumption for $g$ to be  differentiable. One cannot hope in general for $\nabla g$ to  be locally Lipschitz continuous. For instance set $\YY=\R_+$, $L(x,y)=xy - \frac{1}{3} y^{3}$, then  $g(x)=\max_{y\in\RR^+}L(x,y) = \frac{2}{3}x^{3/2}$ with $\nabla g(x)=\sqrt{x}$.
\end{remark}

\paragraph{Comments and rationale of the method} At this stage the principles of our strategy can be made precise. We deal with problems which are dissymetric in structure: the argmax is easily computable or approximable while the block involving the minimizing variable is difficult to handle. This suggests to proceed as follows: one computes a best response mapping (even approximately), the gradient of the value function becomes accessible via formula~\ref{pi:gdiff2} in Proposition~\ref{p:gdiff}, and thus a descent step can be taken. The questions are: {\em which steps are acceptable? Can they be tuned automatically?} This is the object of the next sections.

\subsection{Gradient descent for nonconvex functions with globally H\"olderian gradient}
\label{sec:gradGlobelHolder}
The results of this section are self contained. We consider first the ideal case of a gradient method on a globally H\"older function with known constants, see e.g. \cite{Nest15,Yash16}.
We study Algorithm~\ref{algo:1}, previously presented in \cite{Yash16}
for which we prove sequential convergence. 
\begin{assumption}[Global H\"older regularity]
\label{ass:0f}
$f\colon\R^d \to\RR$ is $C^1$, semi-algebraic  and 
\begin{equation}
\forall x_1, x_2\in \R^d , \,
\quad \| \nabla f(x_1)- \nabla f(x_2)\|\leq \beta\|x_1-x_2\|^\nu, \mbox{ with }\beta>0, \,\nu\in ]0,1].
\end{equation}
\end{assumption}

\begin{algorithm}
\label{algo:0}
\SetKwInOut{Init}{Initialization}
\KwIn{$\nu\in\rzeroun$, $\beta\in\RPP$ and
$\gamma\in]0,(\nu+1)/\beta[$}
\Init{$x_0\in\R^d $}
\For{$n=0,1,\ldots$}
{
$\gamma_n (x_n)= \gamma
\left(\frac{\nu+1}{\beta}\right)^{1/\nu-1}\|\nabla f(x_n)\|^{1/\nu-1}$\\
$x_{n+1} = x_n - \gamma_n(x_n) \nabla f(x_n)$
}
\caption{H\"older gradient method}
\end{algorithm}

\begin{proposition}[Convergence of the H\"older gradient method for nonconvex  functions]
\label{p:gr}
Under Assumption~\ref{ass:0f}, consider a bounded sequence $(x_n)_{n\in\NN}$  generated by
Algorithm~\ref{algo:0}.
Then the following hold:
\begin{enumerate}
\item
\label{p:gro}
the sequence $(f(x_n))_{n \in \NN}$ is nonincreasing and converges,
\item
\label{p:gri}
the sequence $(x_n)_{n\in\NN}$ converges to a critical point
$x^*\in\R^d $ of $f$, i.e., $\nabla f(x^*)=0$,
\item
\label{p:grii}
for every $n\in\NN$ \footnote{This result is essentially present in \cite{Yash16}}, 
\begin{equation*}
\min_{0\leq k\leq n}\|\nabla
f(x_k)\|^{\frac{1}{\nu}+1}\leq 
\left[\frac{f(x_0)-f(x^*)}{\gamma-\gamma^{\nu+1}\left(\frac{\beta}{\nu+1}\right)^\nu}
\left(\frac{\beta}{\nu+1}\right)^{\frac{1-\nu}{\nu}}\right]
\dfrac{1}{n+1} = O\left( \frac{1}{n}\right).
\end{equation*}
Choosing $\gamma = \frac{\nu+1}{\beta}\left(\frac{1}{\nu+1}\right)^{1 / \nu}$, we obtain
\begin{equation*}
\min_{0\leq k\leq n} \|\nabla f(x_k)\|^{\frac{1}{\nu}+1}\leq
\frac{(f(x) - f(x^*)) \beta^{1/\nu} (\nu+1)}{\nu(n+1)}.
\end{equation*}

\end{enumerate}
\end{proposition}
\subsection{The ``Backtrack H\"older'' gradient algorithm and  diagonal backtracking}
\label{sec:backtrackHolder}
In practice, the constants are unknown and the H\"olderian properties are merely
 local. The algorithm we present now (Algorithm~\ref{algo:1}),   is in the spirit of  the classical backtracking method, see e.g., \cite{bertsekas2014constrained}.   
The major difference is that we devise a {\em diagonal backtracking}, to detect both constants $\beta,\nu$ at  once in a single searching pass. 

\begin{assumption}
\label{ass:1f}
$f\colon\R^d \to\RR$ is a $C^1$ semi-algebraic  function such that
$\nabla f$ is locally H\"older.
\end{assumption}
In the following algorithm, $\alpha,\gamma>0$ are step length parameters, $\delta>0$ is a sufficient-decrease threshold and $\rho>0$ balances the search between the unknown exponent $\nu$ and the unknown multiplicative constant $\beta$, see Assumption~\ref{ass:0f}.
\begin{algorithm}[!h]
\label{algo:1}
\SetKwInOut{Init}{Initialization}
\KwIn{$\delta,\alpha\in\zeroun$ and $\gamma,\rho\in\RPP$}
\Init{$x_0\in\R^d $, $k_{-1}=0$}
\For{$n=0,1,\ldots$}{
$k = k_{n-1}$ \\ $\gamma_n(x_n) = \alpha^{k}\min\{1,\|\nabla f(x_n)\|^{\rho k}\}\gamma$\\
\While{
$f(x_n-\gamma_n(x_n) \nabla f(x_n) )> f(x_n)
- \delta\gamma_n(x_n)\|\nabla f(x_n)\|^2$
}
{
$k = k+1$\\
$\gamma_n(x_n) = \alpha^k\min\{1,\|\nabla f(x_n)\|^{\rho k}\}\gamma$\\
}
$k_n = k$\\
$x_{n+1} = x_n - \gamma_n(x_n) \nabla f(x_n)$
}
\caption{Backtrack H\"older gradient method}
\end{algorithm}

The following theorem provides convergence guarantees under local H\"older continuity (Assumption~\ref{ass:1f} for Algorithm~\ref{algo:1}).
\begin{theorem}[Convergence of Backtrack H\"older  for nonconvex functions]
\label{t:lochold}
Under As\-sump\-tion~\ref{ass:1f}, consider  a bounded sequence $(x_n)_{n\in\NN}$ generated by
Algorithm~\ref{algo:1}. 
Then the following hold:
\begin{enumerate}
\item
\label{t:lhe0}
$(\gamma_n)_{n\in\NN}$ is well defined,
\item
\label{t:lheii0}
the sequence $(f(x_n))_{n \in \NN}$ is nonincreasing and converges,
\item
\label{t:lheii}
there exists $x^* \in \R^d $ such that
$x_n\to x^*$ and $\nabla f(x^*)=0,$
\item
\label{t:lheiii} the while loop has a uniform finite bound 
$\bar{k} : = \sup_{n\in\NN} k_n < \pinf$.  Moreover 
\[\min_{0\leq i\leq n} \|\nabla f (x_i)\| = O\left(\frac{1}{n^{\frac{1}{2
+ \rho \bar{k}}}}\right),\]
\item
\label{t:lheiv}
suppose moreover that there exist $\beta\in\RPP$ and $\nu\in]0,1]$ such that
$\nabla f$ is globally $(\beta,\nu)$ H\"older. Then
\begin{equation}
\label{e:boundkn}
\sup_{n\in\NN} k_n\leq 1+\frac{1}{\nu}\max\left\{\frac{\log \left(\frac{(1-\delta)(\nu+1)}{\gamma^\nu\beta}\right)}{
\log(\alpha)},
\frac{1-\nu}{\rho} \right\}.
\end{equation}
\end{enumerate}
\end{theorem}

\begin{remark}[Diagonal backtracking alternatives and comments]
\label{rem:kn}
In the previous theorem, we ask $(k_n)_{n\in\NN}$ to be a nondecreasing sequence and \eqref{e:boundkn}
is actually a bound on the total number of additional calls to the function in the while loop. In practice, this approach  might be too conservative and other strategies may provide much more aggressive steps at the cost of additional calls to the function. We will use two variations to update $(k_n)_{n\in\NN}$:
\begin{itemize}
    \item Initialize $k$ to $0$ for fine tuning  to the price of longer inner loops  (see Algorithm~\ref{algo:minmax_heur1}).
    \item For some iterations, decrease the value of $k$ by $1$ (see Algorithm~\ref{algosink} for example). 
\end{itemize}
The parameters $\alpha, \gamma, \delta$, and $\rho$ are made to tune finely the number of iterations spent on estimating local H\"older constants.
\end{remark}
In Theorem~\ref{t:lochold}\ref{t:lheiv}, we need to suppose that the gradient is {\em globally} H\"older  contrary to Assumption~\ref{ass:1f}. 

\section{Backtrack H\"older for min-max problems}
\label{sec:minmax}

Gathering the results provided in  Section~\ref{sec:backtrackHolder}, we provide here our main 
algorithm (Algorithm~\ref{algo:minmax}).
\begin{algorithm}
\label{algo:minmax}
\SetKwInOut{Init}{Initialization}
\KwIn{$\delta,\alpha\in\zeroun$ and $\gamma,\rho\in\RPP$}
\Init{$x_0\in\R^d $, $y_0 = \arg\max_{z \in \mathcal{Y}} L(x_0,z)$, $k_{-1}=0$}
\For{$n=0,1,\ldots$}{
$k = k_{n-1}$\\
$\gamma_n(x_n) = \gamma\alpha^{k}\min\{1,\|\nabla_x L(x_n,y_n)\|^{\rho k}\}$\\
$x = x_n-\gamma_n(x_n) \nabla_x L(x_n,y_n)$\\
$y = \arg\max_{z \in \mathcal{Y}} L(x,z)$\\
\While{
$L(x,y)> L(x_n,y_n)-
\delta\gamma_n(x_n)\|\nabla_x L(x_n,y_n)\|^2$}
{
$k = k+1$\\
$\gamma_n(x_n)=\gamma\alpha^{k}\min\{1,\|\nabla_x L(x_n,y_n)\|^{\rho k}\}$\\
$x = x_n-\gamma_n(x_n) \nabla_x L(x_n,y_n)$\\
$y = \arg\max_{z \in \mathcal{X}} L(x,z)$
}
$k_n = k$, \quad $x_{n+1} = x$, \quad $y_{n+1} = y$.
}
\caption{Monotone Backtrack H\"older for min-max}
\end{algorithm}

Several comments are in order:

--- The above contains an inner loop whose overhead cost becomes negligible as $n \to \infty$, this allows one for automatic step size tuning. The form of Algorithm~\ref{algo:minmax} is slightly different from Algorithm~\ref{algo:1} to avoid duplicate calls to the max-oracle required both to compute gradients and evaluate functions.

--- As described in Remark~\ref{rem:kn}, the backtracking strategy is one among others and it is  adaptable to different settings. In this min-max case, the cost of the max-oracle may have some impact: either it is costly and extra-flexibility is needed or it is cheap and it can be kept as is. Two examples are provided in Sections~\ref{sec:sink} and~\ref{sec:wgan}.

--- A direct modification of the above method, provides also an algorithm for
\begin{equation}
\label{e:minmin}
 \min_{x\in\R^d }\min_{y\in\YY}L(x,y).
\end{equation}

--- Algorithm~\ref{algo:minmax} is a model algorithm  corresponding to a monotone backtracking approach (i.e., the sequence $(k_n)$ is nondecreasing), but many other variants are possible, see Appendix~\ref{app:D}. 
 Algorithms~\ref{algosink} and~\ref{algo:minmin_armijo} are for the min-min problem, with a non monotone backtracking
and the same guarantees. A   heuristic version is also considered: it is  Algorithm~\ref{algo:minmax_heur1} where an approximation of the argmax is used.

To benchmark our algorithms, we  compare them to Algorithms~\ref{algo:minmin_const}, \ref{algo:minmin_armijo}, 
and~\ref{algo:minmax_const} in Appendix \ref{app:D}, with constant but finely tuned step sizes or with Armijo search. 

\begin{theorem}[Backtrack H\"older for min-max]
\label{t:minmaxbt}
Under Assumptions~\ref{a:LC1} and~\ref{ass:psing}, consider  the
sequences $(x_n)_{n\in\NN}$ and $(y_n)_{n\in\NN}$ generated by
Algorithm~\ref{algo:minmax}.
Suppose that $(x_n)_{n\in\NN}$ is bounded. Then
\begin{enumerate}
\item
 The while loop has a uniform bound, i.e.,  $\sup_{n\in\NN} k_n<\pinf$. 
\item $(x_n)_{n\in\NN}$
converges to $x^*$ in $\R^d $ and $(y_n)_{n \in \NN}$ converges to $y^*\in\YY$, with  $\nabla_x L(x^*,y^*) = 0$ and $y^*=\argmax_{y\in\YY} L(x^*,y)$.
\item  Suppose that there exist $\beta\in\RPP$ and $\nu\in]0,1]$ such that
$\nabla g$ is $(\beta,\nu)$ H\"older everywhere. Then the cost of the while loop is bounded by
\begin{align}
\label{e:boundkn2}
\sup_{n\in\NN} k_n\leq 1+ \frac{1}{\nu}\max\left\{\frac{\log \left(\frac{(1-\delta)(\nu+1)}{\gamma^\nu\beta}\right)}{\log(\alpha)} ,
\frac{1-\nu}{\rho} \right\}.
\end{align}
\end{enumerate}
\end{theorem}

\begin{remark} In  \cite[Proposition~2]{Good14}, the authors mention an algorithm akin to what we proposed, but without backtracking. They insist on the fact that if one had access to the max-oracle, then one would be able to implement a gradient descent by using "sufficiently small updates".
Our theoretical results are an answer to this comment as we offer a quantitative characterization of how small the step should be,  as well as a backtracking estimation technique.
\end{remark}

\section{Numerical experiments}

 We compare our method with constant step size algorithm and Armijo backtracking for the
Generative Adversarial Network (GAN) problem, first using Sinkhorn divergences and second considering Wasserstein adversarial networks. Data lie  in $\R^d  = \RR^2$, the sample size is $N=1024$ and
we consider $x_1, \ldots, x_N\in\R^d $ a fixed sample from a distribution 
$\mathbb{P}_d$, which is a Gaussian mixture, see Figure~\ref{fig:datadistrib}, and $z_1, \ldots, z_N\in\mathcal{Z}$ 
a fixed sample from latent distribution $\mathbb{P}_z=U([0,1] \times [0,1])$, uniform on $\mathcal{Z}$, 
where $\mathcal{Z} = [0,1] \times [0,1]$. 
\begin{wrapfigure}{r}{0.5\linewidth}
\centering
\vspace{-0.8cm}
    \caption{Data distribution $x_1,\ldots,x_N$}
 \includegraphics[width=5cm,height=3.9cm]{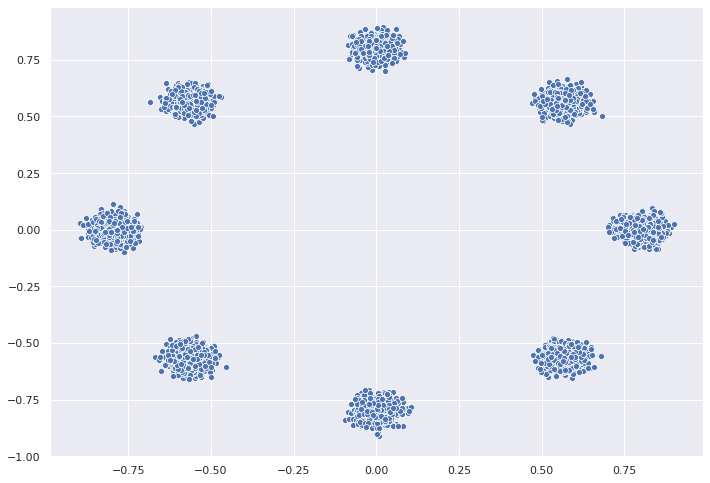}
\label{fig:datadistrib}
\end{wrapfigure}

We consider as {\em generator} $G$, a dense neural network with three hidden layers containing respectively 
$64$, $32$, $16$ neurons with a ReLU activation between each layer.
We write $G : \mathcal{Z} \times \Theta_G \to \mathcal{X}$, with inputs in $\mathcal{Z}$ and parameters $\theta_G \in \Theta_G$ where $\Theta_G = \RR^q$ with $q$ the total number of parameters of the network (2834 in our case). 

\subsection{Sinkhorn GAN}
\label{sec:sink}
We first consider training generative network using Sinkhorn divergences as proposed in \cite{pmlr-v84-genevay18a}. This is a min-min problem which satisfies Assumption~\ref{ass:psing} (see also the remark in Equation~\ref{e:minmin}). Sinkhorn algorithm \cite{sinkhorn1964,cuturi2013sinkhorn} allows us to compute a very precise approximation of the min-oracle required by our algorithm, we use it as an exact estimate. Note that the transport plan for the Sinkhorn divergence is regularized by an entropy term whence the inner minimization problem has a unique solution and the corresponding $p$ is continuous. This is a perfect example to illustrate our ideas. Consider the following probability measures
\begin{align*}
    \bar{\mu} &= \frac{1}{N} \sum_{i=1}^N \delta_{x_i}, &\text{(empirical target distribution)},\\
    \mu(\theta_G) &= \frac{1}{N} \sum_{i=1}^N \delta_{G(z_i,\theta_G)}, &\text{(empirical generator distribution)}.
\end{align*}
We then define the {\em Sinkhorn divergence} between these two distributions.
\begin{align*}
    \mathcal{W}_\epsilon(\bar{\mu},  \mu(\theta_G) ) &=  \min_{P \in \mathbb{R}_+^{N\times N}} \left\{\mbox{Tr}(PC(\theta_G)^T) + \epsilon \sum_{i,j=1}^N P_{ij} \log(P_{ij})\; ;\; P1_N=1_N, P^T1_N=1_N \right\}
\end{align*}
where $\epsilon>0$ is a regularization parameter, $\displaystyle
C(\theta_G)=\left[\|G(z_i,\theta_G)-x_j\|\right]_{i,j}\in\mathbb{R}^{N
\times N}$ is the pairwise distance matrix between target observations
$(x_i)_{1\leq i \leq N}$ and generated observations $(G(z_i,\theta_G))_{1\leq
i\leq N}$. Here Tr is the trace, and $1_N$ is the all-ones vector. The optimum is unique thanks to the entropic regularization and the optimal transportation plan $P$ can be efficiently estimated with an arbitrary precision by Sinkhorn algorithm \cite{sinkhorn1964,cuturi2013sinkhorn}. 

Training our generative network amounts to solving the following  min-min problem
\[
    \min_{\theta_G} \mathcal{W}_\epsilon(\bar{\mu},  \mu(\theta_G) ).
\]

\begin{remark}[Global subanalyticity ensures \L ojasiewicz inequality]
The cost function of the Sinkhorn GAN problem is not 
semi-algebraic due to log. However we never use the logarithm in a neighborhood of $0$ during the optimization process because of its  infinite slope. Hence the loss can actually be seen as globally subanalytic. 
Whence $p$, $g$ are globally subanalytic and the \L ojasiewicz inequality as well as H\"older properties still hold, see  \cite{bochnak1987geometrie,attouch2010proximal, bolte2018nonconvex} for more on this.
\end{remark}

\paragraph{Algorithmic strategies for Sinkhorn GAN}The monotone diagonal backtracking is too conservative for this case, so we use a variant described in Algorithm~\ref{algosink} instead. At each step, the idea is to try to {\em decrease} $k$ of $1$ whenever possible, keeping some sufficient-decrease property valid. Otherwise $k$ is increased as in the monotone method, until sufficient decrease of the value is ensured. This approach is particularly adapted to the Sinkhorn case, because estimating the best response is cheap.

Note that, to propose a fair comparison and keep the same complexity between algorithms, we count each call to the min-oracle, both in the outer and in the inner while loop, as an iteration step.
The parameters used in this experiment for Algorithm~\ref{algosink} are $\gamma = 1$, $\alpha = 0.5$, $\delta = 0.25$, $\delta_+ =0.95$, and $\rho = 0.5$. We compare with Algorithm~\ref{algo:minmin_const} presented in Appendix~\ref{app:D}, which is a constant step size variant, we try with different step size parameters $\gamma \in \{0.01, 0.05, 0.1\}$.
We compare with the standard Armijo backtracking algorithm (see Algorithm~\ref{algo:minmin_armijo} in Appendix~\ref{app:D}) which uses a similar approach as in Algorithm~\ref{algosink} to tune the step size parameter $\gamma_n$, but does not take advantage of the H\"older property. All algorithms are initialized randomly with the same seed.

\begin{figure}[ht]
\centering
\includegraphics[width=0.45\linewidth]{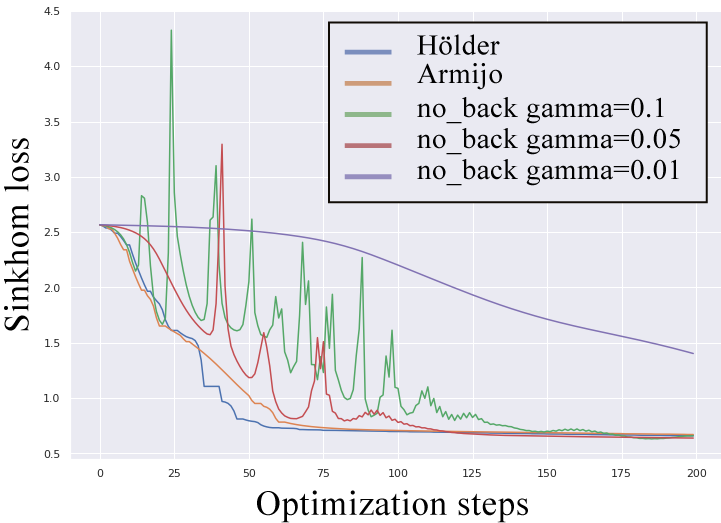}
\qquad
  \includegraphics[width=0.45\linewidth]{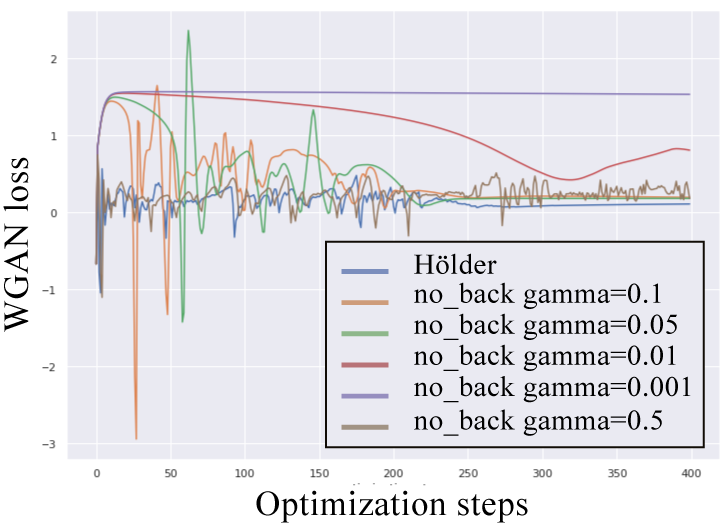}
\caption{Left: Sinkhorn loss with respect to number of Sinkhorn max-oracle evaluation for different gradient step rules. The $x$ axis accounts for all oracle calls, not only the ones used to actually perform gradient steps.
Right: WGAN loss with respect to iteration number for different gradient step rules.}
\label{fig:sinkreswass}
\end{figure}

We observe on the left part of Figure~\ref{fig:sinkreswass} that both H\"older and Armijo backtracking provide decreasing sequence and avoid oscillations. Both algorithms converge faster than the constant step size variant. Furthermore, since our algorithm can take into account the norm of the gradient, the number of intern loop is smaller and that explain why the Non Monotone H\"older backtracking is faster.
  
\subsection{Wasserstein GAN}
\label{sec:wgan}
We treat the Wasserstein GAN (WGAN) heuristically with an approximation of the max-oracle and use Algorithm~\ref{algo:minmax_heur1} in Appendix~\ref{app:D} which matches this setting.

Consider a second neural network, called \emph{discriminator}, $D : \R^d \times \Theta_D \to \mathbb{R} $ with inputs in $\R^d$ and parameters $\theta_D \in \Theta_D$ whose architecture is the same as $G$ (i.e.,  $\Theta_D=\Theta_G$) but with a fullsort activation  between each layer, see \cite{pmlr-v97-anil19a}. We consider the following problem
\[
\min_{\theta_G}
\max_{\theta_D}
\sum_{i=1}^n D(x_i,\theta_D) - \sum_{j=1}^n D(G(z_j,\theta_G),\theta_D).
\]

In order to implement the analogy with Kantorovitch duality in the context of GANs \cite{Arjo17}, one has to ensure that the discriminator $D$ is 1-Lipschitz, when seen as a function of its input neurons. This is enforced using a specific architecture for the discrimintator network $D$. We use Bjork orthonormalization and fullsort activation functions \cite{pmlr-v97-anil19a} which ensure that the network is $1$-Lipschitz without any restriction on its weight parameters $\theta_D$. 

For this problem, we use Algorithm~\ref{algo:minmax_heur1}, which is a heuristic modification of our method designed to deal with the inner max. 
Both the argmax and max are indeed approximated by gradient ascent. Algorithm~\ref{algo:minmax_heur1} then implements the same bactracking idea which  is evaluated on the same to benchmark as in the previous section. Extra discussions are provided in the Appendix.
Doing so, the extra-cost induced by the  while loop becomes negligible and we can find the optimal value of $k$ by exhaustive search. For this reason, in this heuristic context, H\"older backtracking schemes have very little advantage compared to Armijo and we do not report comparison. Detailed investigations for large scale networks is postponed to future research. Since  GAN's training is delicate 
 in practice \cite{Gulr17}, we provide comparison with many step size choices for the constant step size algorithm. 
    Algorithm~\ref{algo:minmax_const} that the difference between Armijo method and ours
 As for Backtrack H\"older min-max, we use parameters $\gamma=1$, $\delta = 0.75$, $\alpha=0.75$, and $\rho=0.20$
for the H\"older backtracking algorithm and constant step size parameter $\gamma \in\{0.001,0.01,0.05,0.1,0.5\}$ for the constant step size variant. All algorithms are initialized randomly with the same seed.

Figure~\ref{fig:sinkreswass} displays our results on the right. The optimal loss equals $0$. One observes that   constant large steps are extremely oscillatory while small steps are stable but extremely slow. Backtrack H\"older takes the best of the two world, oscillates much less and stabilizes closer to the optimal loss value compared to constant step size variants.

\section*{Broader impact}
The authors think that this work is essentially theoretical and that this section does not apply.

{\bf Acknowledgements.} The authors acknowledge the support of ANR-3IA Artificial and Natural Intelligence Toulouse Institute. JB and EP thank Air Force Office of Scientific Research, Air Force Material Command, USAF, under grant numbers FA9550-19-1-7026, FA9550-18-1-0226, and ANR MasDol. JB acknowledges the support of ANR Chess, grant ANR-17-EURE-0010 and ANR OMS.

\bibliographystyle{plain}
\bibliography{my_bib}

\newpage
\appendix

\section{Extra-material}

\begin{definition}[Semi-algebraic sets and functions]\label{def:sa}
\begin{enumerate}
\item A subset $S$ of $\RR^m$ is a \emph{real semi-algebraic set}
if there exist $r$ and $s$ two positive integers, and,
for every $k\in\{1,\ldots,r\}$,  $l\in\{1,\ldots,s\}$,
two real polynomial functions $P_{kl}$, $Q_{kl}\colon\RR^m\to\RR$ such that
\[S=\bigcup_{k=1}^r\bigcap_{l=1}^s \{x\in\R^m: P_{kl}(x)<0,
Q_{kl}(x)=0\}\]
\item 
A function $f:A\subset \R^m\to\R^n$ is semi-algebraic if its graph $\menge{(x,\lambda)\in A\times\RR^{n}}{f(x)=\lambda}$
 is a
semi-algebraic subset of $\RR^{m+n}$.
\end{enumerate}
\end{definition}
For illustrations of this notion in large scale optimization and machine learning we refer to \cite{attouch2010proximal,castera2019inertial}.  

One will also find in this work references, definitions and examples of globally subanalytic sets that are necessary for our proofs on Sinkhorn GANs.
\section{Proofs}

\begin{proofof}{Proposition~\ref{p:gdiff0}}
Let us proceed with the case when $\YY$ is compact; the other case is similar. Let $(x_n)_{n\in\NN}$ be a sequence such that $x_n\to x^*\in\R^d $. We need
to prove that $p(x_n)\to p(x^*)$. For every $n\in\NN$, set $y_n = p(x_n)$ and since $\YY$ is
compact, let $y^*\in\YY$ a cluster point of $(y_n)_{n\in\NN}$. Since
$g$ and $L$ are continuous we have $g(x_n)\to g(x^*)$ and
$L(x_{n_k},y_{n_k})\to L(x^*,y^*)$. Since $p(x_{n_k})=y_{n_k}$ one has
$L(x_{n_k},y_{n_k})\geq L(x,y_{n_k})$ for all $x$ in $\R^d $.
Thus at the limit $L(x^*,y^*)\leq L(x,y^*)$ for all $x$ in $\R^d $.
This implies that $L(x^*,y^*)=g(x^*)$, and so, by uniquennes of the argmax,
$p(x^*) =  y*$. Whence $p$ is continuous.
\end{proofof}

\begin{proofof}{Proposition~\ref{p:gdiff}~\ref{pi:gdiff2}}
This is a consequence of \cite[Theorem~10.31]{Rock98}.
\end{proofof}

\begin{proofof}{Proposition~\ref{p:gdiff}~\ref{pi:gdiff3}}
According to the definition of a semi-agebraic function, we  need to prove that their graph is semi-algebraic.

For $g$: the
\begin{equation*}
\epi g =
\menge{(x,\xi)\in\RR^d\times\RR}{g(x)\leq\xi} =
\menge{(x,\xi)\in\RR^d\times\RR}{(\forall y\in\YY)\quad
L(x,y)\leq\xi}
\end{equation*}
and its complement set is
$\menge{(x,\xi)\in\RR^d\times\RR}{(\exi y\in\YY)\ L(x,y)>\xi}$
which is the projection of
\[\menge{(x,\xi,y)\in\RR^d\times\RR\times\RR^{d'}}{L(x,y)>\xi}\, \bigcap
\,\RR^d \times\RR \times \YY.\]
As a conclusion it is semi-algebraic by Tarski-Seidenberg principle. The same being true for the hypograph, $g$ is semi-algebraic.

For $p$:
$\gra p= \menge{(x,y)\in\R^d \times\YY}{(\forall y'\in\YY)\ L(x,y)\geq
L(x,y')}$. Then $\gra p$ is defined from a first-order formula and the
conclusion follows from~\cite[Theorem~2.6]{Cost99}.
\end{proofof}

\begin{proofof}{Proposition~\ref{p:gdiff}~\ref{pi:gdiff4}} Using Assumption~\ref{ass:psing} {\em H2} and~\ref{pi:gdiff3}, $p$ is continuous and semi-algebraic so using Proposition~\ref{p:phold}, $p$ is locally H\"older. Similarly Assumption~\ref{ass:psing}  {\em H2},~\ref{pi:gdiff2} and~\ref{pi:gdiff3} ensure that $\nabla g$ is also continuous ans semi-algebraic and the result follows again from Proposition~\ref{p:phold}.
\end{proofof}

\begin{proofof}{Proposition~\ref{p:gr}}
Let $n\in\NN$ and set $d_n=\|\nabla f(x_n)\|$. For the clarity of the
proof, the dependence of $\gamma_n$ in $x_n$ is dropped.

Lemma~\ref{l:desc} with $U = \RR^d$ provides
\begin{align}
f(x_{n+1})&\leq f(x_n)+ \scal{d_n}{x_{n+1}-x_n} +
\frac{\beta}{\nu+1}\|x_{n+1}-x_n\|^{\nu+1}\nonumber\\
&\leq f(x_n)- \frac{1}{\gamma_n}\|x_{n+1}-x_n\|^2
+\frac{\beta}{\nu+1}\|x_{n+1}-x_n\|^{\nu+1}\nonumber\\
&
\label{e:fxn2}
= f(x_n)-
\frac{1}{\gamma_n}\big(\|x_{n+1}-x_n\|^2-
\frac{\beta\gamma_n}{\nu+1}\|x_{n+1}-x_n\|^{\nu+1}\big).
\end{align}
By definition of $\gamma_n$ we have
\begin{equation*}
\gamma_n^{1/\nu} = \gamma \left(
\frac{\nu+1}{\beta}\right)^{1/\nu-1}
\|x_{n+1}-x_n\|^{1/\nu-1}
\end{equation*}
and thus
\begin{equation*}
\gamma_n = \gamma^\nu \left(
\frac{\nu+1}{\beta}\right)^{1-\nu}
\|x_{n+1}-x_n\|^{1-\nu}.
\end{equation*}
Set $\delta = 1-\gamma^\nu\left(\frac{\beta}{\nu+1}\right)^\nu$.
Since $\gamma< \frac{\nu+1}{\beta}$ by hypothesis in Algorithm~\ref{algo:0}, we
have $\delta>0$ and we deduce from \eqref{e:fxn2} that
\begin{align}
f(x_{n+1})&\leq f(x_n)-
\frac{1}{\gamma_n}\left(\|x_{n+1}-x_n\|^2
-\gamma^\nu\left(\frac{\beta}{\nu+1}\right)^\nu
\|x_{n+1}-x_n\|^{2}\right)\nonumber\\
&= f(x_n)-
\frac{1-\gamma^\nu\left(\frac{\beta}{\nu+1}\right)^\nu}{\gamma_n}
\|x_{n+1}-x_n\|^2\nonumber\\
&
\label{e:xn21}
= f(x_n)-\frac{\delta}{\gamma_n}
\|x_{n+1}-x_n\|^2.
\end{align}
Hence $(f(x_n))_{n\in\NN}$ is nonincreasing and since $(x_n)_{n\in\NN}$ is bounded, $(f(x_n))_{n\in\NN}$ is also bounded and converges, this proves~\ref{p:gro}.
Since, for all $n$, $\|\nabla f(x_n)\| = \|x_{n+1}-x_n\|/\gamma_n$ and
$(x_n)_{n\in\NN}$ bounded we can apply Theorem~\ref{t:tr} and obtain
that $x_n \to x^*\in\R^d $. Finally, it follows
from \eqref{e:xn21} that $\|x_{n+1}-x_n\|\to 0$ and that $\|x_{n+1}-x_n\|
= \gamma_n\|\nabla f(x_n)\| = \gamma \left(
\frac{\nu+1}{\beta}\right)^{1/\nu-1}\|\nabla f(x_n)\|^{1/\nu}\to 0$.
Hence $x^*$ is a critical point, which proves~\ref{p:gri}.

For $n$ fixed, we have
\begin{equation*}
\frac{\delta}{\gamma_n}\|x_{n+1}-x_n\|^2 =
\delta\gamma_n\|\nabla f(x_n)\|^2 =
\delta\gamma\left(\frac{\nu+1}{\beta}\right)^{ \frac{1}{\nu} -1}\|\nabla
f(x_n)\|^{ \frac{1}{\nu} +1}.
\end{equation*}
Then it follows from \eqref{e:xn21} that
\begin{equation}
\|\nabla f(x_n)\|^{ \frac{1}{\nu}+1}
\leq \frac{1}{\delta\gamma}
\left(\frac{\beta}{\nu+1}\right)^{\frac{1-\nu}{\nu}}
\left(f(x_n)-f(x_{n+1})\right)\label{p:appliqueLoja}
\end{equation}
whence
\begin{align*}
 & (n+1) \min_{k = 0,\ldots, n} \|\nabla f(x_k)\|^{\frac{1}{\nu}+1}\leq \Sum_{k=0}^n\|\nabla f(x_k)\|^{\frac{1}{\nu}+1}\\
 \leq\:\: &
\frac{1}{\gamma-\gamma^{\nu+1}\left(\frac{\beta}{\nu+1}\right)^\nu}
\left(\frac{\beta}{\nu+1}\right)^{\frac{1-\nu}{\nu}}
(f(x_0)-f(x^*)).
\end{align*}
Choosing $\gamma = \frac{\nu+1}{\beta}\left(\frac{1}{\nu+1}\right)^{1 / \nu}$, we obtain
\begin{align*}
    \gamma-\gamma^{\nu+1}\left(\frac{\beta}{\nu+1}\right)^\nu & =  \frac{\nu+1}{\beta}\left(\frac{1}{\nu+1}\right)^{1 / \nu} - \left( \frac{\nu+1}{\beta}\left(\frac{1}{\nu+1}\right)^{1 / \nu} \right)^{\nu+1}\left(\frac{\beta}{\nu+1}\right)^\nu \\
    &=\frac{\nu+1}{\beta} \left(\frac{1}{\nu+1}\right)^{1 / \nu} \left(1 -  \frac{1}{1 + \nu}\right)\\
    &=\frac{\nu+1}{\beta} \left(\frac{1}{\nu+1}\right)^{1 / \nu} \left(\frac{\nu}{1 + \nu}\right).
\end{align*}
from which we deduce
\begin{align*}
(n+1) \min_{0\leq k\leq n}
\|\nabla f(x_k)\|^{\frac{1}{\nu}+1}&\leq
\frac{(f(x) - f(x^*)) (\nu+1)}{\nu} \frac{\left(\frac{\beta}{\nu+1}\right)^{\frac{1-\nu}{\nu}}}{\frac{\nu+1}{\beta} \left(\frac{1}{\nu+1}\right)^{1 / \nu}}\\
&= (f(x) - f(x^*)) \beta^{1/\nu} \frac{\nu+1}{\nu},
\end{align*}
which proves~\ref{p:grii}.
\end{proofof}

\begin{proofof}{Theorem~\ref{t:lochold}}
\ref{t:lhe0} :
For every $n\in\NN$,
by using Taylor expansion, the test implies 
\begin{equation*}
(\forall \tilde{\gamma}\in\RPP)\quad -\tilde{\gamma}\|\nabla f(x_n)\|^2
+ o(\tilde{\gamma}) \leq -\delta\tilde{\gamma}\|\nabla f(x_n)\|^2,
\end{equation*}
thus 
\begin{equation*}
(\forall \tilde{\gamma}\in\RPP)\quad
\frac{o(\tilde{\gamma})}{\tilde{\gamma}}\leq(1-\delta)\|\nabla f(x_n)\|^2
\end{equation*}
which is satisfied for $\tilde{\gamma}$ small.\\
\ref{t:lheii0}:
 It follows from Algorithm~\ref{algo:1} that for every
$n\in\NN$,
\begin{equation}
\label{e:36}
f(x_{n+1})\leq f(x_n)- \frac{\delta}{\gamma_n}\|x_{n+1}-x_n\|^2
\end{equation}
so the descent property holds.

\ref{t:lheii}: One has $\|\nabla f(x_n)\|\leq \frac{1}{\gamma_n} \|x_{n+1}-x_n\|$. 
Since $(x_n)_{n\in\NN}$ is bounded, we
conclude by Theorem~\ref{t:tr}\ref{t:trii} that there exists $x^*\in\RR^d $ such that
$x_n\to x^*$. This follows directly from~\ref{t:lheii0} and \eqref{e:36}.

\ref{t:lheiii}:
Since $f$ is locally H\"older, there exist $U\subset \R^d $, a convex neighborhood of $x^*$,
$\nu\in\rzeroun$, and $\beta\in\RPP$ such that $\nabla f$ is
$(\beta,\nu)$ H\"older on $U$ and $(x_n)_{n\geq N}$ remains in $U$ for
 $N$ sufficiently large. 

Fix any $K \in \NN$ such that
\begin{equation}
\label{e:bounded}
K \geq \max\left\{\frac{\log \left(\frac{(1-\delta)(\nu+1)}{\gamma^\nu\beta}\right)}{ \log(\alpha)\nu},
\frac{1-\nu}{\rho\nu} \right\}
\end{equation}
then we also have
\begin{equation*}
\alpha^{K}\leq \frac{1}{\gamma}\left(\frac{(1-\delta)(\nu+1)}{\beta}\right)^{1/\nu}
\text{and}\quad
\rho K\geq \frac{1}{\nu} -1.
\end{equation*}

We deduce that for any $x \in U$ such that $x- \lambda \nabla f(x) \in U$,
\begin{align}
\label{e:37}
     \lambda^\nu : = \alpha^{K\nu}\min\{1,\|\nabla f(x)\|^{\rho K\nu}\}\gamma^\nu&\leq
    \left(\frac{(1-\delta)(\nu+1)}{\beta}\right)\min\{1,\|\nabla f(x)\|^{\rho K\nu}\}\nonumber \\
    &\leq
    \left(\frac{(1-\delta)(\nu+1)}{\beta}\right)\min\{1,\|\nabla f(x)\|^{1-\nu}\}.
\end{align}
We derive from Lemma~\ref{l:desc} and \eqref{e:37} that for any $x \in U$
\begin{align*}
f(x-\lambda\nabla f(x))&\leq f(x)- \lambda\|\nabla f(x)\|^2 +
\frac{\beta}{\nu+1}\lambda^{\nu+1}\|\nabla f(x)\|^{\nu+1}\\
&\leq f(x)- \lambda\|\nabla f(x)\|^2\nonumber\\
&\qquad +\frac{\beta\lambda }{\nu+1}\frac{(1-\delta)(\nu+1)}{\beta}
\min\{1,\|\nabla f(x)\|^{1-\nu}\}\|\nabla f(x)\|^{\nu+1}.
\end{align*}
Since $\min\{1,\|\nabla f(x)\|^{1-\nu}\}\|\nabla f(x)\|^{\nu+1}\leq \|\nabla f(x)\|^2$, we have for all $x \in U$ such that $x- \lambda \nabla f(x) \in U$,
\begin{align}
\label{e:38}
f(x-\lambda\nabla f(x))&\leq f(x)- \lambda\|\nabla f(x)\|^2+
(1-\delta)\lambda\|\nabla f(x)\|^{2}\nonumber\\
&= f(x)- \delta \lambda\|\nabla f(x)\|^2.
\end{align}
Fix any $N_0\in\NN$ large enough such that $x_N \in U$ for all $N \geq N_0$. Suppose that $K = k_{N_0}$ satisfies \eqref{e:bounded}, then for all $N \geq N_0$ we may consider equation \eqref{e:38} with $x = x_N$, $\lambda = \gamma_N$, noting that $x_{N+1} = x_N - \gamma_N \nabla f(x_N) \in U$. This is exactly the negation of the condition to enter the while loop of Algorithm~\ref{algo:1}. Hence, by a simple recursion, the algorithm never enters the while loop after step $N_0$ and we have $k_N = k_{N_0}$ for all $N \geq N_0$. On the other hand, if $k_{N_0}$ does not satisfy \eqref{e:bounded}, then since $k$ is incremented by $1$ at each execution of the while loop, using the fact that \eqref{e:bounded} implies \eqref{e:38}, it must hold that 
\begin{align*}
    k_N \leq 1+ \max\left\{\frac{\log \left(\frac{(1-\delta)(\nu+1)}{\gamma^\nu\beta}\right)}{ \log(\alpha)\nu},
\frac{1-\nu}{\rho\nu} \right\}.
\end{align*}
In all cases, we have using monotonicity of $k_N$ in $N$ that for all $N \in \NN$,
\begin{align}
    \label{eq:boundK}
    k_N \leq 1+\max\left\{k_{N_0}, \frac{\log \left(\frac{(1-\delta)(\nu+1)}{\gamma^\nu\beta}\right)}{ \log(\alpha)\nu}  ,
\frac{1-\nu}{\rho\nu}\right\},
\end{align}
hence $(k_n)_{n\in\NN}$ is bounded. 

Now we use \eqref{e:36} and~\ref{t:lheii0} which ensures that 
\begin{align*}
    \frac{\|x_{n+1} - x_n\|^2}{\gamma_n} = \gamma_n \|\nabla f(x_n)\|^2
\end{align*}
is summable and thus tends to $0$ as $n \to \infty$, whence, either $(\gamma_n)_{n \in \NN}$ or $(\|\nabla f(x_n)\|^2)_{n \in \NN}$ tends to zero. Using the fact that $(k_n)_{n\in\NN}$ is bounded, in any case we have, $\nabla f(x_n) \to 0$ as $n \to \infty$.

It follows for $n$ large enough that $ \|\nabla f(x_n)\|\leq 1$, from the while loop condition and the fact that $\bar{k} := \sup_{n \in \NN} k_n < + \infty$, that
\begin{equation*}
\delta\alpha^{\bar{k}}
\|\nabla f(x_n)\|^{2+\rho \bar{k}}\gamma\leq\delta\alpha^{k_n}
\|\nabla f(x_n)\|^{2+\rho k_n}\gamma=\delta\gamma_n(x_n)\|\nabla f(x_n)\|^2\leq f(x_n)-f(x_{n+1}).
\end{equation*}
Using the convergence of $(f(x_n))_{n\in\NN}$, and 
by summing the previous equation and taking the minimum, we obtain that $\min_{0\leq i\leq n} \|\nabla f (x_i)\|^{2 + \rho \bar{k}} = O(1/n)$.

\ref{t:lheiv}: The result follows from \eqref{eq:boundK} with $k_{N_0} = 0$, since in this case the same reasoning can be applied for all $N \in \NN$ with $U = \RR^d$.
\end{proofof}

\begin{proofof}{Theorem~\ref{t:minmaxbt}}
Recall that $g(\cdot) = \max_{y\in\YY} L(\cdot,y)$ and $p(\cdot) = \argmax_{y\in\YY}
L(\cdot,y)$.
It follows from Proposition~\ref{p:gdiff} that for every $n\in\NN$,
$\nabla_x L(x_n,y_n) = \nabla g(x_n)$.  We derive from Proposition~\ref{pi:gdiff4} that $\nabla g$ is locally H\"older. 
It turns out that Algorithm~\ref{algo:minmax} applied to $L$ is the same as Algorithm~\ref{algo:1} applied to $g$.
Thus Theorem~\ref{t:lochold} ensures the convergence of
$(x_n)_{n\in\NN}$ to a critical point $x^*\in\R^d $ of $g$. Furthermore,
it follows from the continuity of $p$ that $y_n\to y^* = p(x^*)$.
We conclude that $(x_n,y_n)_{n\in\NN}$ converges to a critical point of
$L$, satisfying $y^*=\argmax_{y\in\YY} L(x^*,y)$.
Finally, since for every $n\in\NN$, $\nabla g(x_n)=\nabla_x
L(x_n,y_n)$, we conclude by Theorem~\ref{t:lochold}\ref{t:lheiv}.
\end{proofof}

\section{Lemmas}

\begin{proposition}[Continuity and semi-algebraicity implies H\"older continuity]\cite{bochnak1987geometrie}
\label{p:phold}
Let $f\colon\R^d \to\R^{d'}$ be a semi-algebraic continuous function, then
$f$ is locally H\"older, i.e.,  for all compact set $K$,
\begin{equation*}
\exi \beta\in\RPP, \exi \nu\in\rzeroun, \forall x,y\in K, 
\quad \|f(x)-f(y)\|\leq \beta\|x-y\|^\nu.
\end{equation*}
\end{proposition}

We recall below the \L ojasiewicz inequality, see e.g  \cite{kurdyka1998gradients} and references therein. 

\begin{definition}[\L ojasiewicz inequality]
\label{def:kl}
A differentiable function $f\colon\RR^n\to\RR$ has the \L ojasiewicz
property at $x^*\in\RR^n$ if there exist $\eta,C\in\RPP$
and $\theta\in\zeroun$ such that
for all $x\in B(x^*,\eta)$,
the following inequality is satisfied
\begin{equation*}
C|f(x)-f(x^*)|^{\theta}\leq\|\nabla f(x)\|.
\end{equation*}
In this case the set $B(x^*,\eta)$ is called a \L ojasiewicz ball.
\end{definition}

\begin{lemma}[H\"older Descent Lemma]
\cite[Lemma~1]{Yash16}
\label{l:desc}
Let $U\subset X$ be a nonempty convex set, let $f\colon\R^d \to\RR$ be a $C^1$ function, let $\nu\in]0,1]$, and let
$\beta\in\RPP$. Suppose that
\begin{equation*}
    \forall (x,y)\in U^2,\quad
    \|\nabla f(x)- \nabla f(y)\|\leq \beta\|x-y\|^\nu.
\end{equation*}
Then
\begin{equation*}
  \forall (x,y)\in U^2,\quad
f(x)\leq f(y) + \scal{\nabla f(x)}{x-y}
+ \frac{\beta}{\nu+1}\|y-x\|^{\nu+1}.
\end{equation*}
\end{lemma}

\begin{lemma}[Controlled descent]
\label{l:le1}
Let $\delta,\theta\in\zeroun$, let $C\in\RPP$
and let $x,y,x^*\in\R^d $
Suppose that the following hold:
\begin{enumerate}[label=\rm(\alph*)]
\item $f(x)\geq f(x^*)$ et $f(y)\geq f(x^*)$,
\item
$f(y)\leq f(x)-\frac{\delta}{\gamma} \|y-x\|^{2}$,
\item $\|\nabla f(x)\|\leq \frac{1}{\gamma} \|y-x\|$,
\item
$C(f(x)-f(x^*))^\theta\leq\|\nabla f(x)\|$.
\end{enumerate}
Then
\begin{equation*}
\delta C(1-\theta)\|y-x\|\leq (f(x)-f(x^*))^{1-\theta}
-(f(y)-f(x^*))^{1-\theta}.
\end{equation*}
\end{lemma}
\begin{proof}
First,  if $y=x$, then the inequality holds trivially. Second, if $f(x) = f(x^*)$, then by the first two items, $y = x$ and the inequality holds also. Second, also. Hence we may suppose that $f(x) - f(x^*) > 0$ and $y \neq x$.
We have
\begin{equation*}
\frac{C\gamma}{\|y-x\|} \leq \frac{C}{\|\nabla f(x)\|}
\leq (f(x)-f(x^*))^{-\theta}.
\end{equation*}
By concavity of $s\mapsto s^{1-\theta}$, we have
\begin{align*}
(f(x)-f(x^*))^{1-\theta}-(f(y)-f(x^*))^{1-\theta}&\geq
(1-\theta)(f(x)-f(x^*))^{-\theta}(f(x)-f(y))\\
&\geq \frac{C\gamma(1-\theta)}{\|y-x\|}
\frac{\delta}{\gamma}\|y-x\|^{2}\\
&\geq \delta C(1-\theta) \|y-x\|,
\end{align*}
which concludes the proof.
\end{proof}

We slightly adapt the recipe from \cite{Bolt14a} and add a trap argument from \cite{Atto13}.
\begin{theorem}[Recipe for convergence \cite{Bolt14a} and the trapping phenomenon]
\label{t:tr}
Let $f\colon\R^d \to\RR$ be a $C^1$ function and let $\delta\in\zeroun$.
Consider $(x_n)_{n\in\NN}$ in $\R^d $
and $(\gamma_n)_{n\in\NN}$ in $\RPP$ that satisfies
\begin{enumerate}[label={\rm[\alph*]}]
\item
\label{t:tra}
$(\forall n\in\NN)$ $f(x_{n+1})
\leq f(x_n)- \dfrac{\delta}{\gamma_n} \|x_{n+1}-x_n\|^2$,
\item
\label{t:trb}
$(\forall n\in\NN)$
$\|\nabla f(x_n)\|\leq \dfrac{1}{\gamma_{n}}\|x_{n+1}-x_n\|$.
\end{enumerate}
Then the following results hold:
\begin{enumerate}
\item
\label{t:tri}
Assume that there exist $x^*\in\R^d$, $\theta\in ]0,1[$, $\eta,C\in\RPP$, and $N\in \NN$ such that
\begin{align*}
    & C|f(x)-f(x^*)|^\theta\leq \|\nabla f(x)\|, \forall x\in B(x^*,\eta),\\
 & x_N\in B(x^*,\eta/2),\\
& |f(x_N)-f(x^*)|^{1-\theta} <\delta C(1-\theta)\eta/2.
\end{align*}
If $f(x_n)\geq f(x^*)$ for all $n\in\NN$, then $(x_n)_{n\geq N}$ lies entirely in $B(x^*,\eta)$.
\item
\label{t:trii}
Suppose that $f$ is semi-algebraic. Then if $(x_n)_{n\in\NN}$ has a
cluster point $x^*\in\RR^d$, then it converges to $x^*$.
\end{enumerate}
\end{theorem}
\begin{proof}

\ref{t:tri}:
By assumption, we have 
$x_N\in B(x^*,\eta/2)$ and
\begin{equation*}
|f(x_N)-f(x^*)|^{1-\theta} <\delta C(1-\theta)\eta/2.
\end{equation*}
It follows from Lemma~\ref{l:le1} that
\begin{equation*}
\delta C(1-\theta)\|x_{N+1}-x_N\|\leq
(f(x_N)-f(x^*))^{1-\theta}-(f(x_{N+1})-f(x^*))^{1-\theta}.
\end{equation*}
Let us prove by strong induction that for every $n\geq N$, $x_n\in
B(x^*,\eta)$. Assume $n\geq N+1$ and suppose that for every integer
$N\leq k\leq n-1$, $x_k\in B(x^*,\eta)$. Lemma~\ref{l:le1} yields
\begin{equation*}
(\forall k\in[N,\ldots, n-1])\quad
\delta C(1-\theta)\|x_{k+1}-x_k\|\leq
(f(x_k)-f(x^*))^{1-\theta}-(f(x_{k+1})-f(x^*))^{1-\theta}.
\end{equation*}
By summing we have
\begin{align}
\delta C(1-\theta)\Sum_{k=N}^{n-1}\|x_{k+1}-x_k\| &\leq (f(x_N)-f(x^*))^{1-\theta}
-(f(x_{n})-f(x^*))^{1-\theta}\nonumber\\
&\leq (f(x_N)-f(x^*))^{1-\theta}<\delta C(1-\theta)\eta/2.
\label{eq:summabilityIncrements}
\end{align}
Since
\begin{align*}
\|x_{n}-x^*\|&\leq
\Sum_{k=N}^{n-1}\|x_{k+1}-x_{k}\|+\|x_{N}-x^*\|< \eta/2+\eta/2= \eta
\end{align*}
we have $x_n\in B(x^*,\eta)$.
We have proved that for every $n\geq N$, $x_n \in B(x^*,\eta)$.

\ref{t:trii}: Since $(f(x_n))_{n\in\NN} $ is nonincreasing by~\ref{t:tra}, we deduce
that $f(x_n)\to f(x^*)$ and $f(x_n)\geq f(x^*)$. Since $f$ is
semi-algebraic, $f$ has the \L ojasiewicz property at $x^*$ \cite{kurdyka1998gradients}. Hence, let
us define $\theta$, $C$, and $\eta$ as in Definition~\ref{def:kl} relative to $x^*$. Since $x^*$ is a cluster point and $f(x_n)\to f(x^*)$, there exists $N$ as in \ref{t:tri} above. 

For every integer $n\geq N$,
it follows from \eqref{eq:summabilityIncrements} that 
\begin{equation*}
\delta C(1-\theta)\Sum_{k=N}^{n-1}\|x_{k+1}-x_k\|\leq
(f(x_N)-f(x^*))^{1-\theta}
\leq (f(x_0)-f(x^*))^{1-\theta}<\pinf
\end{equation*}
hence the serie converges, increments are summable and $(x_k)_{k\in\NN}$ converges to $x^*$.
\end{proof}

\begin{remark}[Convergence and semi-algebraicity] (a) Note that when $f$ is semi-algebraic, we have in fact an alternative, for any sequence:
\begin{itemize}
    \item either $\|x_n\|\to\infty$
    \item or $(x_n)_{n\in\NN}$ converges to a critical point $x^*$.
\end{itemize}
Indeed if we are not in the diverging case, there is a cluster point $x^*$ which must be a critical point. Whence we are in the situation of (ii) above.\\
(b) If $f$ is, in addition, coercive, i.e., $\lim_{\|x\|\to+\infty} f(x)=+\infty$, each H\"older gradient sequence converges to a critical point since the first alternative is not possible because $(f(x_k))_{k \in \NN}$ is non increasing so that $(x_k)_{k \in \NN}$ is bounded.
\end{remark}

\section{Numerical Experiments : Complements}
\label{app:D}

In practice, it can be difficult to calculate the
argmax (or the argmin) or to perform rigorously the internal while loop, we propose two 
algorithms to simplify this  implementation aspect. We also present the constant step size algorithm that we use to assess the efficiency of our method.

\subsection{Sinkhorn GAN}
Sinkhorn GAN is a min-min problem, thus our model must  be slightly adapted.  First, we start with Algorithm~\ref{algo:minmin_const} below which is a constant step size algorithm. Due to the specific setting of Sinkhorn problem, the argmin may be computed exactly.
\begin{algorithm}[!h]
\label{algo:minmin_const}
\SetKwInOut{Init}{Initialization}
\KwIn{$\gamma\in\RPP$}
\Init{$x_0\in\R^d $}
\For{$n=0,1,\ldots$}{
$y_n = \argmin_{y\in\YY} L(x_n,y)$\\
$x_{n+1} = x_n - \gamma \nabla_x L(x_n,y_n)$
}
\caption{Constant step size gradient method for min-max}
\end{algorithm}
The next algorithm is a  Backtrack H\"older method for the min-min problem. For gaining efficiency, we introduce a new rule in Algorithm~\ref{algosink}, which maintains the sufficient decrease property, without the monotonicity of $(k_n)_{n\in\NN}$.
\begin{algorithm}[!h]
\label{algosink}
\SetKwInOut{Init}{Initialization}
\KwIn{$N\in\NN$,$\gamma,\rho\in\RPP$, and $\alpha,\delta,\delta_+\in\zeroun$}
\Init{$k_{-1}=1$, $n=0$, and $x_0\in\R^d $}
\While{$n <N$}{
$k = k_{n-1}$ \\
$n=n+1$\\ 
$\gamma_n(x_n) = \alpha^k\min\{1,\|\nabla L(x_n,y_n)\|^{\rho k}\}\gamma$\\ 
\If{$\min_y L(x_n-\gamma_n(x_n)\nabla f(x_n),y)
< \min_y L(x_n,y)
- \delta^+\gamma_n(x_n)\|\nabla f(x_n)\|^2$}
{
 $k=k-1$ 
}
\While{
$\min_y L(x_n-\gamma_n(x_n) \nabla f(x_n),y)> \min_y L(x_n,y)-\delta\gamma_n(x_n)\|\nabla L(x_n,y_n)\|^2$
}
{
$k = k+1$\\ 
$n=n+1$\\ 
$\gamma_n(x_n) = \alpha^k\min\{1,\|\nabla L(x_n,y_n)\|^{\rho k}\}\gamma$\\ 
}
$k_n = k$\\
$x_{n+1} = x_n - \gamma_n(x_n) \nabla L(x_n,y_n)$
}
\caption{Non Monotone Backtrack H\"older for min-max}
\end{algorithm}

We also present an Armijo search process for this problem in Algorithm~\ref{algo:minmin_armijo}. 
It has a structure similar to the ``Non Monotone H\"older Backtrack" but with a much less clever  update for $\gamma_n$.
\begin{algorithm}
\label{algo:minmin_armijo}
\SetKwInOut{Init}{Initialization}
\KwIn{$N\in\NN$,$\gamma,\rho\in\RPP$, and $\alpha,\delta,\delta_+\in\zeroun$}
\Init{$k_{-1}=1$, $n=0$, and $x_0\in\R^d $}
\While{$i <N$}{
$k = k_{n-1}$ \\
$n=n+1$\\ 
$\gamma_n(x_n) = \alpha^k\min\{1,\|\nabla L(x_n,y_n)\|^{\rho k}\}\gamma$\\ 
\If{$\min_y L(x_n-\gamma_n(x_n)\nabla f(x_n),y)
< \min_y L(x_n,y)
- \delta^+\gamma_n(x_n)\|\nabla f(x_n)\|^2$}
{
 $k=k-1$ 
}
\While{
$\min_y L(x_n-\gamma_n(x_n) \nabla f(x_n),y)> \min_y L(x_n,y)-\delta\gamma_n(x_n)\|\nabla L(x_n,y_n)\|^2$
}
{
$k = k+1$\\ 
$n=n+1$\\ 
$\gamma_n(x_n) = \alpha^k\min\{1,\|\nabla L(x_n,y_n)\|^{\rho k}\}\gamma$\\ 
}
$k_n = k$\\
$x_{n+1} = x_n - \gamma_n(x_n) \nabla L(x_n,y_n)$
}
\caption{Non Monotone Armijo for min-max}
\end{algorithm}

\newpage
\subsection{Wasserstein GAN}

As explained in  Section~\ref{sec:wgan}, this problem does not formally match our setting. In particular, the argmax cannot be computed fast, so we  use a gradient ascent to provide an approximation expressed by using the sign $\approx$. We also provide a constant step  size method (Algorithm~\ref{algo:minmax_const}) to benchmark our algorithm.
\begin{algorithm}
\label{algo:minmax_const}
\SetKwInOut{Init}{Initialization}
\KwIn{$\gamma\in\RPP$}
\Init{$x_0\in\R^d $}
\For{$n=0,1,\ldots$}{
$y_n \approx \argmax_{y\in\YY} L(x_n,y)$\\
$x_{n+1} = x_n - \gamma \nabla_x L(x_n,y_n)$
}
\caption{Heuristic gradient method for min-max with constant step size}
\end{algorithm}

Besides, since the max is not easily accessible,  we modify the while loop by using $y_n$ instead of the exact argmax to validate the sufficient decrease. This approach gives  Algorithm~\ref{algo:minmax_heur1}.
\begin{algorithm}
\label{algo:minmax_heur1}
\SetKwInOut{Init}{Initialization}
\KwIn{$\gamma,\rho\in\RPP$ and $\delta,\alpha\in\zeroun$}
\Init{$x_0\in\R^d $}
\For{$n=0,1,\ldots$}{
$y_n \approx \argmax_{y\in\YY} L(x_n,y)$\\
$k = 0$\\
$\gamma_n(x_n) = \gamma$\\
\While{
$ L(x_n-\gamma_n(x_n) \nabla_x L(x_n,y_n),y_n)> L(x_n,y_n)-
\delta\gamma_n(x_n)\|\nabla_x L(x_n,y_n)\|^2$}
{
$k = k+1$\\
$\gamma_n(x_n)=\gamma\alpha^{k}\min\{1,\|\nabla_x L(x_n,y_n)\|^{k\rho}\}$\\
}
$k_n = k$\\
$x_{n+1} = x_n - \gamma_n(x_n) \nabla_x L(x_n,y_n)$
}
\caption{Heuristic H\"older Backtrack for min-max}
\end{algorithm}

\end{document}